\documentclass[10pt]{article}

\usepackage{times}
\usepackage{graphicx} 
\graphicspath{{./pictures/}}
\DeclareGraphicsExtensions{.pdf,.jpeg,.png}

\usepackage{subfigure} 

\usepackage{algpseudocode}
\usepackage{algorithm}
\usepackage{algorithmicx}

\usepackage{hyperref}

\usepackage[table]{xcolor}
\usepackage{booktabs}
\usepackage{multirow}
\usepackage{multicol}

\usepackage[cmex10]{amsmath} 
\usepackage{amsthm}
\usepackage{amssymb}

\usepackage{url}

\newcommand\tstrut{\rule{0pt}{2.4ex}}
\newcommand\bstrut{\rule[-1.0ex]{0pt}{0pt}}
\newcommand{\tr}{\operatorname{tr}}
\newcommand{\prox}{\operatorname{prox}}
\newcommand{\argmin}{\operatornamewithlimits{argmin}}
\newtheorem{theorem}{Theorem}[section]

\newtheorem{proposition}[theorem]{Proposition}

\newtheorem{definition}[theorem]{Definition}

\newcolumntype{a}{>{\columncolor{gray!35}}c}

\title{Learning Multiple Visual Tasks while Discovering their Structure}

\author{Carlo Ciliberto  \thanks{
       Laboratory for Computational and Statistical Learning, Istituto Italiano di Tecnologia,
       Via Morego, 30,
       16100, Genova, Italy, ({\tt cciliber@mit.edu})} \ 
 \and
Lorenzo Rosasco  $^*$ \thanks{
       DIBRIS, Universit\`a di Genova,
       Via Dodecaneso, 35,
       16146, Genova, Italy, ({\tt lrosasco@mit.edu})} \ 
 \and
Silvia Villa $^*$
}
\date{}

\begin{document} 
\maketitle

\begin{abstract}

Multi-task learning is a natural approach for computer vision applications that require the simultaneous solution of several distinct but related problems, e.g. object detection, classification, tracking of multiple agents, or denoising, to name a few. The key idea is that exploring task relatedness (structure) can lead to improved performances.

In this paper, we propose and study a novel sparse, non-parametric approach 
exploiting the theory of Reproducing Kernel Hilbert Spaces for vector-valued functions.
We develop a suitable regularization framework which can be formulated as a convex optimization problem,  
and  is provably solvable using an alternating minimization approach. Empirical tests  
show that the proposed method compares favorably to state of the art techniques and further allows to recover
interpretable structures, a  problem of interest in its own right.

\end{abstract}

\section{Introduction}

Several problems in computer vision and image processing, such as  object detection/classification, image denoising, inpainting etc., require solving multiple learning tasks at the same time. In such settings a natural question is to ask whether it could be beneficial to solve all the tasks jointly, rather than separately. This idea is at the basis of the field of multi-task learning, where the joint solution of different problems has the potential to exploit tasks relatedness (structure) to improve learning. Indeed, when knowledge about task relatedness is available, it can be profitably incorporated in multi-task learning approaches for example by designing  suitable embedding/coding schemes, kernels or regularizers, see \cite{micchelli04,evgeniou05,alvarez12,fergus10,lozano10}.

The more interesting case, when knowledge about the tasks structure is not known a priori, has been the subject
of recent studies. Largely influenced by the success of sparsity based methods, a common approach has been that of considering linear models for each task coupled with suitable parameterization/penalization enforcing task relatedness, for example encouraging the selection of features simultaneously important for all tasks~\cite{argyriou08} or for specific subgroups of related tasks~\cite{jacob08,jayaraman14,zhong12,kang11,hwang11,kumar12}. Other linear methods adopt hierarchical priors or greedy approaches to recover the taxonomy of tasks~\cite{salakhutdinov11,torralba04}. A different line of research has been devoted to the development of non-linear/non-parametric approaches using kernel methods -- either from a Gaussian process~\cite{alvarez12,zhong12} or a regularization perspective~\cite{alvarez12,dinuzzo11}.

This paper follows this last line of research, tackling in particular two issues only partially addressed in previous works. The first is the development of a regularization framework to learn and exploit the tasks structure, which is not only important for prediction, but also for interpretation. Towards this end, we propose and study a family of matrix-valued reproducing kernels, parametrized so to enforce sparse relations among tasks. A novel algorithm dubbed Sparse Kernel MTL is then proposed considering a Tikhonov regularization approach.
The second contribution is to provide a sound computational framework to solve the corresponding minimization problem. While we follow a fairly standard alternating minimization approach, unlike most previous work we  can exploit  results in convex optimization to prove the  convergence of the considered procedure. 
The latter has an interesting interpretation where supervised and unsupervised learning steps are alternated: first, given a structure, multiple tasks are learned jointly,  then the structure is updated. 
%
We support the proposed method with an experimental analysis both on synthetic and real data, including classification and detection datasets.  The obtained results show that Sparse Kernel MTL 
can achieve state of the art performances while unveiling the structure describing tasks relatedness. 

The paper is organized as follows: in Sec.~\ref{sec:model} we provide some background and notation in order to motivate and introduce the Sparse Kernel MTL model. In Sec.~\ref{sec:optimization} we discuss an alternating minimization algorithm to provably solve the learning problem proposed. Finally, we discuss empirical evaluation in Sec.~\ref{sec:empirical}.







{\bf Notation.} With $S^n_{++} \subset S^n_+ \subset S^n \subset \mathbb{R}^{n \times n}$ we denote respectively the space of positive definite, positive semidefinite (PSD) and symmetric $n \times n$ real-valued matrices. $O^n$ denotes the space of orthonormal $n \times n$ matrices. For any $M\in\mathbb{R}^{n \times m}$, $M^\top$ denotes the transpose of $M$. For any PSD matrix $A\in S_+^n$, $A^\dagger \in S_+^n$ denotes the pseudoinverse of $A$.  We denote by $I_n\in S_{++}^n$ the $n \times n$ identity matrix. We use the abbreviation l.s.c. to denote lower semi-continuous functions (i.e. functions with closed sub-level sets)~\cite{boyd04}.


\section{Model}\label{sec:model}

We formulate the problem of solving multiple learning tasks as that of learning a vector-valued function whose output components correspond to individual predictors. We consider the framework originally introduced in~\cite{micchelli04} where the well-known concept of Reproducing Kernel Hilbert Space is extended to spaces of vector-valued functions. In this setting the set of tasks relations has a natural characterization in terms of a positive semidefinite matrix. By imposing a sparse prior on this object we are able to formulate our model, Sparse Kernel MTL, as a kernel learning problem designed to recover the most relevant relations among the tasks.

In the following we review basic definitions and results from the theory of Reproducing Kernel Hilbert Spaces that will allow in Sec.~\ref{sec:sparse_relation_learning} to motivate and introduce our learning framework. In Sec.~\ref{sec:previous_work} we briefly draw connections of our method to previously prosed multi-task learning approaches.

\subsection{Reproducing Kernel Hilbert Spaces for Vector-Valued Functions}\label{sec:RKHSvv}

We consider the problem of learning a function $f: \mathcal{X} \to \mathcal{Y}$ from a set of empirical observations $\{(x_i,y_i)\}_{i=1}^n$ with $x_i \in \mathcal{X}$ and $y_i\in\mathcal{Y} \subseteq \mathbb{R}^T$. This setting includes learning problems such as vector-valued regression ($\mathcal{Y} = \mathbb{R}^T$), multi-label/detection for $T$ tasks ($\mathcal{Y} = \{0,1\}^{T}$) or also $T$-class classification (where we adopt the standard one-vs-all approach mapping the $t$-th class label to the $t$-th element $e_t$ of the canonical basis in $\mathbb{R}^T$). Following the work of Micchelli and Pontil~\cite{micchelli04}, we adopt a Tikhonov regularization approach in the setting of Reproducing Kernel Hilbert Spaces for vector-valued functions (RKHSvv). RKHSvv are the generalization of the well-known RKHS to the vector-valued setting and maintain most of the properties of their scalar counterpart. In particular, similarly to standard RKHS, RKHSvv are uniquely characterized by an operator-valued kernel:

\begin{definition}
Let $\mathcal{X}$ be a set and $(\mathcal{H},\langle\cdot,\cdot\rangle_{\mathcal{H}})$ be a Hilbert space of functions from $\mathcal{X}$ to $\mathbb{R}^T$. A symmetric, positive definite, matrix valued function $\Gamma : \mathcal{X} \times \mathcal{X} \to \mathbb{R}^{T \times T}$ is called a reproducing kernel for $\mathcal{H}$ if for all $x \in \mathcal{X}, c \in \mathbb{R}^T$ and $f \in \mathcal{H}$ we have that $\Gamma(x,\cdot)c \in \mathcal{H}$ and the following reproducing property holds: $\langle f(x), c\rangle_{\mathbb{R}^T} = \langle f,\Gamma(x,\cdot)c\rangle_{\mathcal{H}}.$
\end{definition}

Analogously to the scalar setting, a Representer theorem holds, stating that the solution to the regularized learning problem
\begin{equation}\label{eq:learning_problem}
    \underset{f\in\mathcal{H}}{\text{minimize}} \ \  \frac{1}{n} \sum_{i=1}^n V(y_i,f(x_i)) + \lambda \|f\|_\mathcal{H}^2
\end{equation}
is of the form $f(\cdot) = \sum_{i=1}^n \Gamma(\cdot,x_i)c_i$ with $c_i\in\mathbb{R}^T$, $\Gamma$ the matrix-valued kernel associated to the RKHSvv $\mathcal{H}$ and $V:\mathcal{Y} \times \mathbb{R}^T \to \mathbb{R}_+$ a loss function (e.g. least squares, hinge, logistic, etc.) which we assume to be convex. We point out that the setting above can also account for the case where not all task outputs $y_i = (y_{i1},\dots,y_{iT})^\top$ associated to a given input $x_i$ are available in training. Such situation would arise for instance in multi-detection problems in which supervision (e.g. presence/absence of an object class in the image) is provided only for a few tasks at the time.

\subsubsection{Separable Kernels}\label{sec:separable_kernels}

Depending on the choice of operator-valued kernel $\Gamma$, different structures can be enforced among the tasks; this effect can be observed by restricting ourselves to the family of {\it separable} kernels. Separable kernels are matrix-valued functions of the form $\Gamma(x,x') = k(x,x')A$, where $k:\mathcal{X}\times\mathcal{X}\to\mathbb{R}$ is a scalar reproducing kernel and $A \in S_+^T$ a $T \times T$ positive semidefinite (PSD) matrix. Intuitively, the scalar kernel characterizes the individual tasks functions, while the matrix $A$ describes how they are related. Indeed, from the Representer theorem we have that solutions of problem~\eqref{eq:learning_problem} are of the form $f(\cdot) = \sum_{i=1}^n k(\cdot,x_i) A c_i$ with the $t$-th task being $f_t(\cdot) = \sum_{i=1}^n k(\cdot,x_i) \langle A_t, c_t \rangle_{\mathbb{R}^T}$, a scalar function in the RKHS $\mathcal{H}_k$ associated to kernel $k$. As shown in~\cite{evgeniou05}, in this case the squared norm associated to the separable kernel $kA$ in the RKHSvv $\mathcal{H}$, can be written as
\begin{equation}\label{eq:norm}
    \|f\|_\mathcal{H}^2 = \sum_{t,s}^T A^{\dagger}_{ts} \langle f_t,  f_s \rangle_{\mathcal{H}_k}
\end{equation}
with $A^\dagger_{ts}$ the $(t,s)$-th entry of $A$'s pseudo-inverse.

Eq.~\eqref{eq:norm} shows how $A$ can model the structural relations among tasks by directly coupling predictors: for instance, by setting $A^\dagger= I_T + \gamma (\mathbf{1}\mathbf{1}^\top)/T$, with $\mathbf{1}\in\mathbb{R}^T$ the vector of all $1$s, we have that the parameter $\gamma$ controls the variance $\sum_{t=1}^T \|\bar{f} - f_t\|_{\mathcal{H}_k}^2$ of the tasks with respect to their mean $\bar{f}=\frac{1}{T} \sum_{t=1}^T f_t$. If we have access to some notion of similarity among tasks in the form of a graph with adjacency matrix $W\in S^T$, we can consider the regularizer $\sum_{t,s=1}^T W_{ts} \|f_t - f_s\|_{\mathcal{H}_k}^2 + \gamma \sum_{t}^T \|f_t\|_{\mathcal{H}_k}^2$ which corresponds to setting $A^\dagger = L + \gamma I_T$ with $L$ the graph Laplacian induced by $W$. We refer the reader to~\cite{evgeniou05} for more examples of possible choices for $A$ when the tasks structure is known.


\subsection{Sparse Kernel Multi Task Learning}\label{sec:sparse_relation_learning}

When a-priori knowledge of the problem structure is not available, it is desirable to learn the tasks relations directly from the data. In light of the observations of 
Sec.~\ref{sec:separable_kernels}, a viable approach is to parametrize the RKHSvv $\mathcal{H}$ in problem~\eqref{eq:learning_problem} with the 
associated separable kernel $kA$ and to optimize jointly with respect to both $f\in \mathcal{H}$ and $A \in S_+^T$. In the following we show how this problem corresponds to that of identifying a set of latent tasks and to combine them in order to form the individual predictors. By enforcing a sparsity prior on the set of such possible combinations, we then propose the Sparse Kernel MTL model, which is designed to recover only the most relevant tasks relations. In Sec.~\ref{sec:previous_work} we discuss, from a modeling perspective, how our framework is related to the current multi-task learning literature.


\subsubsection{Recovering the Most Relevant Relations}\label{sec:model_srl}

From the Representer theorem introduced in Sec.~\ref{sec:RKHSvv} we know that a candidate solution $f: \mathcal{X} \to \mathbb{R}^T$ to problem~\eqref{eq:learning_problem} can be parametrized in terms of the maps $k(\cdot,x_i)$, by a structure matrix $A\in S_{+}^T$ and a set of coefficient vectors $c_1,\dots,c_n \in \mathbb{R}^T$ such that $f(\cdot) = \sum_{i=1}^n k(\cdot,x_i)Ac_i$. If now we consider the $t$-th component of $f$ (i.e. the predictor of the $t$-th task), we have that
\begin{equation}\label{eq:latent}
    f_t(\cdot) = \sum_{i=1}^n k(\cdot,x_i) \langle A_t, c_i \rangle_{\mathbb{R}^T} = \sum_{s=1}^T A_{ts} g_s(\cdot)
\end{equation}
where we set $g_s(\cdot) = \sum_{i=1}^n k(\cdot,x_i)c_{is} \in \mathcal{H}_k$ for $s\in\{1,\dots,T\}$ and $c_{is}\in\mathbb{R}$ the $s$-th component of $c_i$. Eq.~\eqref{eq:latent} provides further understanding on how $A$ can enforce/describe the tasks relations: The $g_s$ can be interpreted as elements in a dictionary and each $f_t$ factorizes as their linear combination. Therefore, any two predictors $f_t$ and $f_{t'}$ are implicitly coupled by the subset of common $g_s$. 

We consider the setting where the tasks structure is unknown and we aim to recover it from the available data in the form of a structure matrix $A$. Following a denoising/feature selection argument, our approach consists in imposing a sparsity penalty on the set of possible tasks structures, requiring each predictor $f_t$ to be described by a small subset of $g_s$. Indeed, by requiring most of $A$'s entries to be equal to zero, we implicitly enforce the system to recover only the most relevant tasks relations. The benefits of this approach are two-fold: on the one hand it is less sensitive to spurious statistically non-significant tasks-correlations that could for instance arise when few training examples are available. On the other hand it provides us with interpretable tasks structures, which is a problem of interest in its own right and relevant, for example, in cognitive science~\cite{lake10}.

Following the de-facto standard choice of $\ell_1$-norm regularization to impose sparsity in convex settings, the {\it Sparse Kernel MTL} problem can be formulated as
\begin{multline}\label{eq:sparse_relation_learning}
    \underset{f\in\mathcal{H}, A\in S_{++}^T}{\text{minimize}} \ \ \frac{1}{n} \sum_{i=1}^n V(y_i,f(x_i)) \ + \\
    \lambda (\|f\|_\mathcal{H}^2 + \epsilon \tr(A^{-1}) + \mu \tr(A) + (1-\mu) \ \|A\|_{\ell_1})
\end{multline}
where $\|A\|_{\ell_1}=\sum_{t,s}|A_{ts}|$, $V:\mathcal{Y} \times \mathbb{R}^T \to \mathbb{R}_+$ is a loss function and $\lambda>0$, $\epsilon>0$, and  $\mu\in[0,1]$ regularization parameters. Here $\mu \in [0,1]$ regulates the amount of desired entry-wise sparsity of $A$ with respect to the low-rank prior $tr(A)$ (indeed notice that for $\mu = 1$ we recover the low-rank inducing framework of~\cite{argyriou08,zhang10}). This prior was empirically observed (see \cite{argyriou08,zhang10}) to indeed encourage information transfer across tasks; the sparsity term can therefore be interpreted as enforcing such transfer to occur only between tasks that are strongly correlated. Finally the term $\epsilon \tr(A^{-1})$ ensures the existence of a unique solution (making the problem strictly convex), and can be interpreted as a preconditioning of the problem (see Sec.~\ref{sec:unsupervised_step}).

Notice that the term $\|f\|_{\mathcal{H}}^2$ depends on both $f$ and $A$ (see Eq.~\ref{eq:norm}), thus making problem~\eqref{eq:sparse_relation_learning} non-separable in the two variables. However, it can be shown that the objective functional is jointly convex in $f$ and $A$ (we refer the reader to the Appendix for a proof of convexity, which extends results in~\cite{argyriou08} to our setting). This will allow in Sec.~\ref{sec:optimization} to derive an optimization strategy that is guaranteed to converge to a global solution.

\subsubsection{Previous Work on Learning the Relations among Tasks}\label{sec:previous_work}

Several methods designed to recover the tasks relations from the data can be formulated using our notation as joint learning problems in $f$ and $A$. Depending on the expected/desired tasks-structure a set of constraints $\mathcal{A} \subseteq S_{++}^T$ can be imposed on $A$ when solving a joint problem as in~\eqref{eq:sparse_relation_learning}:
\begin{itemize}
\item {\bf Multi-task Relation Learning}~\cite{zhang10}. In~\cite{zhang10}, the relaxation $\mathcal{A} = \{A | \tr(A)\leq1\}$ of the low-rank constraint is imposed, enforcing the tasks $f_t$ to span a low-dimensional subspace in $\mathcal{H}_k$. This method can be shown to be approximately equivalent to~\cite{argyriou08}.
\item {\bf Output Kernel Learning}~\cite{dinuzzo11}. Rather than imposing a hard constraint, the authors penalize the structure matrix $A$ with the squared Frobenius norm $\|A\|_F^2$.
\item {\bf Cluster Multi-task Learning}~\cite{jacob08}. Assuming tasks to be organized into distinct clusters, in~\cite{jacob08} a learning scheme to recover such structure is proposed, which consists of imposing a suitable set of spectral constraints $\mathcal{A}$ on $A$. We refer the reader to the supplementary material for further details.
\item {\bf Learning Graph Relations}~\cite{argyriouboh}. Following the interpretation in~\cite{evgeniou05} reviewed in Sec.~\ref{sec:separable_kernels} of imposing similarity relations among tasks in the form of a graph, in~\cite{argyriouboh} the authors propose a setting where a (relaxed) Graph Laplacian constraint is imposed on $A$.
\end{itemize}

\section{Optimization}\label{sec:optimization}

Due to the clear block variable structure of Eq.~\eqref{eq:sparse_relation_learning} with respect to $f$ and $A$, we propose an alternating minimization approach (see Alg.~\ref{alg:bcd}) to iteratively solve the Sparse Kernel MTL problem by keeping fixed one variable at the time. This choice is motivated by the fact that for a fixed $A$, problem~\eqref{eq:sparse_relation_learning} reduces to the standard multi-task learning problem~\eqref{eq:learning_problem}, for which several well-established optimization strategies have already been considered~\cite{alvarez12,micchelli04,evgeniou05,minh11}. The alternating minimization procedure can be interpreted as iterating between steps of supervised learning (finding the $f$ that best fits the input-output training observations) and unsupervised learning (finding the best $A$ describing the tasks structure, which does not involve the output data).

\subsection{Solving w.r.t. $f$ (Supervised Step)}

Let $A\in S_{++}^T$ be a fixed structure matrix. From the Representer theorem (see Sec.~\ref{sec:RKHSvv}) we know that the solution of problem~\eqref{eq:learning_problem} is of the form $f(\cdot) = \sum_{i=1}^n k(\cdot,x_i) A c_i$ with $c_i \in \mathbb{R}^T$. Depending on the specific loss $V$, different methods can be employed to find such coefficients $c_i$. In particular, for the least-square loss a closed form solution can be derived by taking the coefficient vector $c = (c_1^\top,\dots,c_n^\top)^\top \in \mathbb{R}^{nT}$ to be~\cite{alvarez12}:
\begin{equation}   
c = (A \otimes K + \lambda I_{nT})^{-1}y
\end{equation}
where $K \in S_+^n$ is the empirical kernel matrix associated to $k$ the scalar kernel, $y \in \mathbb{R}^{nT}$ is the vector concatenating the training outputs $y_1,\dots,y_n \in \mathbb{R}^T$ and $\otimes$ denotes the Kronecker product. A faster and more compact solution was proposed in~\cite{minh11} by adopting Sylvester's method.

\begin{algorithm}[t]
   \caption{\textsc{Alternating Minimization}}
   \label{alg:bcd}
\begin{algorithmic}
   \State {\bfseries Input:} $K$ empirical kernel matrix, $y$ training outputs, $\delta$ tolerance, $V$ loss, $\lambda,\mu,\epsilon$ hyperparameters, $S$ objective functional of problem~\eqref{eq:sparse_relation_learning}.
   \State {\bfseries Initialize:} $f_0=0$, $A_0 = I_T$ and $i=0$ 
   \Repeat
   
   \State $f_{i+1} \gets$ \textsc{SupervisedStep} $(V,K,y,A_{i},\lambda)$
    \State $A_{i+1} \gets$ \textsc{SparseKernelMTL}$(K,f_{i+1},\mu,\epsilon)$
     \State $i \gets i+1$

   \Until{$| S(f_{i+1},A_{i+1})-S(f_{i},A_{i}) | < \delta$}
\end{algorithmic}
\end{algorithm}

\subsection{Solving w.r.t the Tasks Structure (Unsupervised Step)}\label{sec:unsupervised_step}

Let $f$ be known in terms of its coefficents $c_1,\dots,c_n\in\mathbb{R}^T$. Our goal is to find the structure matrix $A\in S_{++}^T$ that minimizes problem~\eqref{eq:sparse_relation_learning}. Notice that each task $f_t$ can be written as $f_t(\cdot) = \sum_{i=1}^n k(\cdot,x_i) \langle A_t,c_i \rangle_{\mathbb{R}^T}= \sum_{i=1}^n k(\cdot,x_i)b_{i,t}$ with $b_{i,t} = \langle A_t,c_i\rangle_{\mathbb{R}^T}$. Therefore, from eq.~\eqref{eq:norm} we have
\begin{equation}\label{eq:eq}
    \|f\|_\mathcal{H}^2 = \sum_{t,s}^T A_{ts}^{-1} \langle f_t, f_s \rangle_{\mathcal{H}_k}
    = \sum_{t,s}^T \sum_{i,j} A_{ts}^{-1} k(x_i,x_j) b_{it}b_{js}
\end{equation}
where we have used the reproducing property of $\mathcal{H}_k$ for the last equality. Eq.~\eqref{eq:eq} allows to write the norm induced by the separable kernel $kA$ in the more compact matrix notation $\|f\|_\mathcal{H}^2 = \tr(B^\top K B A^{-1})$, where $B\in\mathbb{R}^{n \times T}$ is the matrix with $(i,t)$-th element $B_{it} = b_{it}$.

Under this new notation, problem~\eqref{eq:sparse_relation_learning} with fixed $f$ becomes
\begin{equation}\label{eq:actual_srl}
    \underset{A \in S_{++}^T}{\text{min.}} \tr(A^{-1} (B^\top K B+\epsilon I_T) ) + \mu \tr(A) + (1-\mu) \ \|A\|_{\ell_1}
\end{equation}
from which we can clearly see the effect of $\epsilon$ as a preconditioning term for the tasks covariance matrix $B^\top K B$.



By employing recent results from the non-smooth convex optimization literature, in the following we will describe an algorithm to optimize the Sparse Kernel MTL problem.

\subsubsection{Primal-dual Splitting Algorithm}
First order proximal splitting algorithms have been successfully applied to 
solve convex composite optimization problems, that can be written as the sum of a smooth 
component with nonsmooth ones \cite{BauCom11}. 
They proceed by splitting, i.e. by activating each term appearing in the sum individually.
The iteration usually consists of a gradient descent-like step determined by the smooth component,
and various proximal steps induced by the nonsmooth terms \cite{BauCom11}. 
In the following we will describe one of such methods, derived in \cite{Bang13,Con13}, 
to solve the Sparse Kernel MTL problem in eq.~\eqref{eq:actual_srl}.
The proposed method is primal-dual, in the sense that it also provides an additional
dual sequence solving the associated dual optimization problem.
We will rely on the sum structure of the objective function, that can be written as
$G(\cdot) + H_1(\cdot) + H_2(L(\cdot))$, with 
$G(A) = \lambda\mu \tr(A)$, $H_1(A) = \lambda(1-\mu)\|A\|_{\ell_1}$ and $H_2(A) =\lambda\epsilon \tr(A^{-1}) + i_{S_{++}^T}(A)$, where $i_{S_{++}^T}$ is the indicator function of a $S_{++}^T$ ($0$ on the set $+\infty$ outside) and enforces the hard constraint $A\in S_{++}^T$. 
$L$ is a linear operator defined as $L(A) = M A M$, where we have set $M = (B^\top K B + \epsilon I_T)^{-1/2}$. 
We recall here that a square root of a PSD matrix $P \in S_{+}^T$ is a PSD matrix $M\in S_+^T$ such that $P = M M$.
Note that $G$ is smooth with Lipschitz continuous gradient, $L$ is a linear operator and both $H_1$ and 
$H_2$ are functions for which the proximal operator can be computed in closed form.
We recall that the proximity operator at a point $y\in\mathbb{R}^m$ of a proper, convex and l.s.c. function $H: \mathbb{R}^m \to \mathbb{R} \cup \{+\infty\}$, 
is defined as
\begin{equation}\label{eq:proxdef}
    \prox_{H} (y) = \argmin_{x \in \mathbb{R}^m} \Big\{H(x) + \frac{1}{2} \| x - y \|^2\Big\}.
\end{equation}

It is well known that for any $\eta>0$, the proximal map of the $\ell_1$ norm $\eta \|\cdot\|_{\ell_1}$ is the so-called {\it soft-thresholding} operator $S_\eta(\cdot)$, which can be computed in closed form. The following result provides an explicit closed-form solution also for the proximal map of $H_2$.

\begin{proposition}\label{prop:prox_h2}
Let $Z\in S^T$ with eigendecomposition $Z = U \Sigma U^\top$ with $U \in O^T$ orthonormal matrix and $\Sigma \in S^T$ diagonal. Then
\begin{equation}\label{eq:prox_h2}
    \prox_{H_2}(Z) =  \argmin_{A \in S_{++}^T} \Big\{\tr(A^{-1}) + \frac{1}{2} \|A-Z\|_F^2\Big\}.
\end{equation}
can be computed in closed form as $\prox_{H_2} (Z) = U \Lambda U^\top$ with $\Lambda\in S_{++}^T$ diagonal matrix with $\Lambda_{tt}$ the only positive root of the polynomial $p(\lambda) = \lambda^3 - \lambda^2 \Sigma_{tt} - 1$ with $\lambda \in \mathbb{R}$.
\end{proposition}

\begin{proof}
Note that $H_2$ is convex and lsc. Therefore the proximity operator is well-defined and the functional 
in~\eqref{eq:prox_h2} has a unique minimizer. Its gradient is $-A^{-2} + A - Z$, therefore, the first order condition for a matrix $A$ to be a minimizer is
\begin{equation}\label{eq:first_order_condition}
A^3 - A^2Z - I_T = 0
\end{equation}
We show that it is possible to find $\Lambda \in S_{++}^T$ diagonal such that $A_* = U \Lambda U^\top$ solves eq.~\eqref{eq:first_order_condition}. Indeed, for $A$ with same set of eigenvectors $U$ as $Z$, we have that eq.~\eqref{eq:first_order_condition} becomes $U(\Lambda^3 - \Lambda^2 \Sigma - I_T)U^\top = 0$, which is equivalent to the set of $T$ scalar equations $\lambda^3 - \lambda^2 \Sigma_{tt} - 1 = 0$ for $t\in \{1,\dots,T\}$ and $\lambda\in\mathbb{R}$. Descartes rule of sign~\cite{struik86} assures that for any $\Sigma_{tt} \in \mathbb{R}$ each of these polynomials has exactly one positive root, which can be clearly computed in closed form.
\end{proof}

\begin{algorithm}[t]
   \caption{\textsc{Sparse Kernel MTL}}
   \label{alg:srl}
\begin{algorithmic}
   \State {\bfseries Input:} $K \in S_{+}^n$, $B \in \mathbb{R}^{n \times T}$, $\delta$ tolerance, $0 \leq \mu \leq 1, \epsilon > 0$ hyperparameter.
   \State {\bfseries Initialize:} $A_0,D_0 \in S_{++}^T$, $M = (B^\top K B + \epsilon I_T)^{-1/2}$, $\sigma = \| M\|^2$ squared maximum eigenvalue of $M$. $i=0$
   \Repeat

       \State $A_{i+1} \gets \prox_{\frac{1-\mu}{\sigma} \|\cdot\|_{\ell_1}} (A_i - \frac{1}{\sigma} (\mu I_T + M D_i M) )$

       \State $P \gets D_i + \frac{1}{\sigma} M (2A_{i+1} - A_i) M$

       \State $D_{i+1} \gets P - \prox_{\sigma H_2} (\sigma P)$

       \State $i \gets i+1$

   \Until{$ \|A_{i+1} - A_{i} \|_F< \delta$ and $ \|D_{i+1} - D_{i}\|_F < \delta$}
\end{algorithmic}
\end{algorithm}

We have the following result as an immediate consequence.
\begin{theorem}[Convergence of Sparse Kernel MTL, \cite{Bang13,Con13}]\label{thm:bangcon}
Let $k$ be a scalar kernel over a space $\mathcal{X}$, $x_1,\dots,x_n\in\mathcal{X}$ a set of points and $f: \mathcal{X} \to \mathbb{R}^T$ a function characterized by a set of coefficients $b_1,\dots,b_n\in\mathbb{R}^T$ so that $f(\cdot) = \sum_{i=1}^n k(\cdot,x_i) b_i$. Set $K\in S_{+}^n$ to be the empirical kernel matrix associated to $k$ and the points $\{x_i\}_{i=1}^n$ and $B \in \mathbb{R}^{n \times T}$ the matrix whose $i$-th row corresponds to the (transposed) coefficient vector $b_i$.

Then, any sequence of matrices $A_t$ produced by Algorithm~\eqref{alg:srl} converges to a global minimizer of the Sparse Kernel MTL problem~\eqref{eq:sparse_relation_learning} (or, equivalently, to \eqref{eq:actual_srl}) for fixed $f$. Furthermore, the sequence $D_t$ converges to a solution of the dual problem of~\eqref{eq:actual_srl}.
\end{theorem}

\subsection{Convergence of Alternating Minimization}

We additionally exploit the sum structure and the regularity properties of the objective functional 
in \eqref{eq:sparse_relation_learning} to prove  convergence of the alternating minimization 
scheme to a global minimum. We rely on the results in \cite{tseng01}. In particular, the following
result is a direct application of Theorem~4.1 in that paper. 
\begin{theorem}\label{thm:tseng} Under the same assumptions as in Theorem~\ref{thm:bangcon}, the 
sequence $(f_i,A_i)_{i\in\mathbb{N}}$  generated by Algorithm~\ref{alg:bcd}  is a 
minimizing sequence for Problem~\ref{eq:sparse_relation_learning} and converges to its unique 
solution.
 \end{theorem}
 \begin{proof}
Let $S$ denote the objective function in \eqref{eq:sparse_relation_learning}.
First note that the level sets of $S$ are compact due to the presence of the  term $\epsilon\tr(A^{-1}) + \mu \tr(A)$
and that $S$ is continuous on each level set. 
Moreover, since $S$ is regular at each point in the interior of the domain and is convex, \cite[Theorem~4.1(c)]{tseng01}
implies that each cluster point of $(f_{i},A_i)_{i\in\mathbb{N}}$ is the unique minimizer of $S$. Then, the sequence itself is convergent
and is minimizing by continuity.
\end{proof}

\subsubsection{A Note on Computational Complexity \& Times}

Regarding the computational costs/number of iterations required for the convergence of the whole Alg.~\ref{alg:bcd}, up to our knowledge the only results available on rates for Alternating Minimization are in~\cite{beck11}. Unfortunately these results hold only for smooth settings. 
Notice however that each iteration of Alg~\ref{alg:srl} is of the order of $O(T^3)$, (the eigendecomposition of A being the most expensive operation) and its convergence rate 
is $O(1/k)$ with $k$ equal to the number of iterations.
Hence, Alg.~\ref{alg:srl} is not affected by the number $n$ of training samples. On the contrary, the supervised step in Agl.~\ref{alg:bcd} (e.g. RLS or SVM) typically requires the inversion of the kernel matrix $K$ (or some approximation of its inverse) whose complexity heavily depends on $n$ (order of $O(n^3)$ for inversion). Furthermore, the product $BKB^\top$ costs $O(n^2T)$ which, since $n>>T$, is more expensive than Alg.~\ref{alg:bcd}. Thus, with respect to $n$ SKMTL scales exactly as methods such as [2,7,24].

\section{Empirical Analysis}\label{sec:empirical}

We report the empirical evaluation of SKMTL on artificial and real datasets.
We have conducted experiments on both artificially generated and real dataset to assess the capabilities of the proposed Sparse Kernel MTL method to recover the most relevant relations among tasks and exploit such knowledge to improve the prediction performance.

\begin{figure}[t]
  \begin{center}
            \includegraphics[width=0.8\columnwidth]{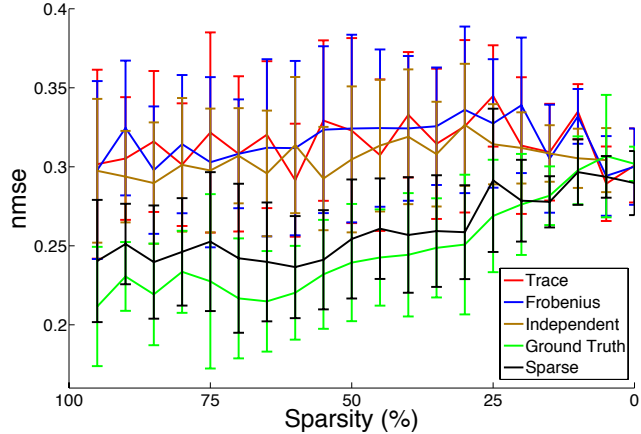}
    \caption{Generalization performance (nMSE and standard deviation) of different multi-task methods with respect to the sparsity of the task structure matrix.}\label{fig:sparsity}
  \end{center}
\end{figure}

\subsection{Synthetic Data}
We considered an artificial setting that allows us to control the tasks structure and in particular the actual sparsity of the tasks-relation matrix. We generated synthetic datasets of input-output pairs $(x,y) \in \mathbb{R}^d \times \mathbb{R}^T$ according to linear models of the form $y^\top = x^\top U A + \epsilon$ where $U\in\mathbb{R}^{d \times T}$ is a matrix with orthonormal columns, $A \in S_+^T$ is the task structure matrix and $\epsilon$ is zero-mean Gaussian noise with variance $0.1$. The inputs $x\in\mathbb{R}^d$ were sampled according to a Gaussian distribution with zero mean and identity covariance matrix. We set the input space dimension $d=100$ for our experiments.

In order to quantitatively control the sparsity level of the tasks-relation matrix, we randomly generated $A$ so that the ratio between its support (i.e. the number of non-zero entries) and the total number of entries would vary between $0.1$ ($90\%$ sparsity) and $1$ (no sparsity). A Gaussian noise with zero mean and variance $1/10$ of the mean value of the non-zero entries in $A$ was sampled to corrupt the structure matrix entries (hence, the model $A$ was never ``really'' sparse). This was done to reproduce a more realistic scenario.

We generated multiple models and corresponding datasets for different sparsity ratios and number of tasks $T$ ranging from $5$ to $20$. For each dataset we generated respectively $50$ samples for training and $100$ for test. We performed multi-task regression using the following methods: single task learning (STL) as baseline, Multi-task Relation Learning~\cite{zhang10} (MTRL), Output Kernel Learning~\cite{dinuzzo11} (OKL), our Sparse Kernel MTL (SKMTL) and a fixed task-structure multi-task regression algorithm solving problem~\eqref{eq:learning_problem} using the ground truth (GT) matrix $A$ (after noise corruption) for regularization. We chose least-square loss and performed model selection with five-fold cross validation.

\begin{figure}[t]
\begin{center}
    \includegraphics[height=.3\columnwidth]{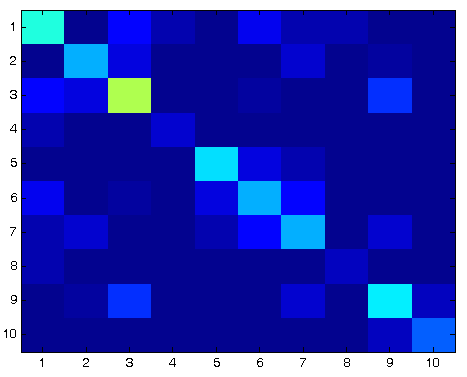}\quad\quad
    \includegraphics[height=.3\columnwidth]{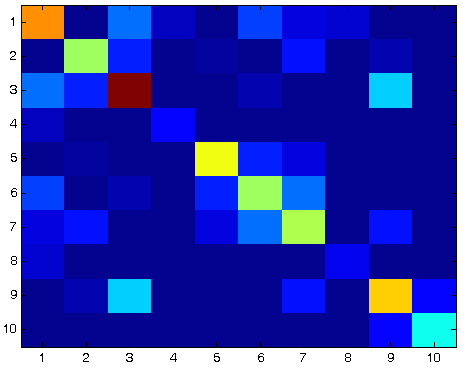}
    \caption{Structure matrix $A$. True (Left) and recovered by Sparse Kernel MTL (Right). We report the absolute value of the entries of the two matrices. The range of values goes from 0 (Blue) to 1 (Red)}
    \label{fig:recovered}
\end{center}
\end{figure}

In Figure~\ref{fig:sparsity} we report the normalized mean squared error (nMSE) of tested method with respect to decreasing sparsity ratios. It can be noticed that knowledge of the true $A$ (GT) is particularly beneficial when the tasks share few relations. This advantage tends to decrease as the tasks structure becomes less sparse. Interestingly, both the MTRL and OKL method do not provide any advantage with respect to the STL baseline since we did not design $A$ to be low-rank (or have a fast eigenvalue decay). On the contrary, the SKMTL method provides a remarkable improvement over the STL baseline. 

We point out that the large error bars in the plot are due to the high variability of the nMSE with respect to the different (random) linear models $A$ and number of tasks $T$. The actual improvement of the SKMTL over the other methods is however significant.

The results above suggest that, as desired, our SKMTL method is actually recovering the most relevant relations among tasks. In support of this statement we report in Figure~\ref{fig:recovered} an example of the true (uncorrupted) and recovered structure matrix $A$ in the case of $T=10$ and $50\%$ sparsity. As can be noticed, while the actual values in the entries of the two matrices are not exactly the same, their supports almost coincide, showing that SKMTL was able to recover the correct tasks structure.

\subsection{15-Scenes}

We tested SKMTL in a multi-class classification scenario for visual scene categorization, the $15$-scenes dataset\footnote{http://www-cvr.ai.uiuc.edu/ponce\_grp/data/}. The dataset contains images depicting natural or urban scenes that have been organized in $15$ distinct groups and the goal is to assign each image to the correct scene category. It is natural to expect that categories will share similar visual features. Our aim was to investigate whether these relations would be recovered by the SKMTL method and result beneficial to the actual classification process.

We represented images in the dataset with LLC coding~\cite{wang10}, trained multi-class classifiers on $50$, $100$ and $150$ examples per class and tested them on $500$ samples per class. We repeated these classification experiments $20$ times to account for statistical variability.
 
In Table~\ref{tab:15scenes} we report the classification accuracy of the multi-class learning methods tested: STL (baseline), Multi-task Feature Learning (MTFL)~\cite{argyriou08}, MTRL, OKL and our SKMTL. For all methods we used a linear kernel and least-squares loss as plug-in classifier. Model selection was performed by five-fold cross-validation.

\begin{table}[t]
\begin{center}
\begin{tabular}{lccc}
    &  \multicolumn{3}{c}{\bf Accuracy (\%) per}  \tstrut \bstrut \\
    &  \multicolumn{3}{c}{\bf \# tr. samples per class}  \tstrut \bstrut \\
    & $50$ & $100$ & $150$ \tstrut \bstrut \\
    \specialrule{.1em}{.05em}{.0em}

    \multirow{2}{*}{\bf STL}                    & $72.23$ & $76.61$ & $79.23$ \tstrut \bstrut \\
                                                & \cellcolor{gray!35} $\pm 0.04$ & \cellcolor{gray!35} $\pm 0.02$ & \cellcolor{gray!35} $\pm 0.01$ \tstrut \bstrut \\

    \multirow{2}{*}{\bf MTFL~\cite{argyriou08}} & $73.23$& $77.24$ & $80.11$ \tstrut \bstrut \\
                                                & \cellcolor{gray!35} $\pm 0.08$& \cellcolor{gray!35} $\pm 0.05$ & \cellcolor{gray!35} $\pm 0.03$ \tstrut \bstrut \\
    \multirow{2}{*}{\bf MTRL~\cite{zhang10}}   & $73.13$& $77.53$ & $80.21$ \tstrut \bstrut \\
                                                & \cellcolor{gray!35} $\pm 0.08$& \cellcolor{gray!35} $\pm 0.04$ & \cellcolor{gray!35} $\pm 0.05$ \tstrut \bstrut \\
    \multirow{2}{*}{\bf OKL~\cite{dinuzzo11}}   & $72.25$& $77.06$ & $80.03$ \tstrut \bstrut \\
                                                & \cellcolor{gray!35} $\pm 0.03$& \cellcolor{gray!35} $\pm 0.01$ & \cellcolor{gray!35} $\pm 0.01$ \tstrut \bstrut \\
    \multirow{2}{*}{\bf SKMTL}                 & $\mathbf{73.50}$& $\mathbf{78.23}$ & $\mathbf{81.32}$ \tstrut \bstrut \\  
                                                & \cellcolor{gray!35} $\pm 0.11$& \cellcolor{gray!35} $\pm 0.06$ & \cellcolor{gray!35} $\pm 0.08$ \tstrut \bstrut \\  
    \end{tabular}
\end{center}
\caption{Classification results on the $15$-scene dataset. Four multi-task methods and the single-task baseline are compared.}
\label{tab:15scenes}
\end{table}

As it can be noticed, the SKMTL consistently outperforms all other methods. A possible motivation for this behavior, similarly to the synthetic scenario, is that the algorithm is actually recovering the most relevant relations among tasks and using this information to improve prediction. In support of this interpretation, in Figure~\ref{fig:structure} we report  the relations recovered by SKMTL in graph form. An edge between two scene categories $t$ and $s$ was drawn whenever the value of the corresponding entry $A_{ts}$ of the recovered structure matrix was different from zero. Noticeably SKMTL seems to identify a clear group separation between natural and urban scenes. Furthermore, also within these two main clusters, not all tasks are connected: for instance office scenes are not related to scenes depicting the exterior of buildings or mountain scenes are not connected to images featuring mostly flat scenes such as highways or coastal regions.

\begin{figure}[t]
  \begin{center}
            \includegraphics[width=0.8\columnwidth]{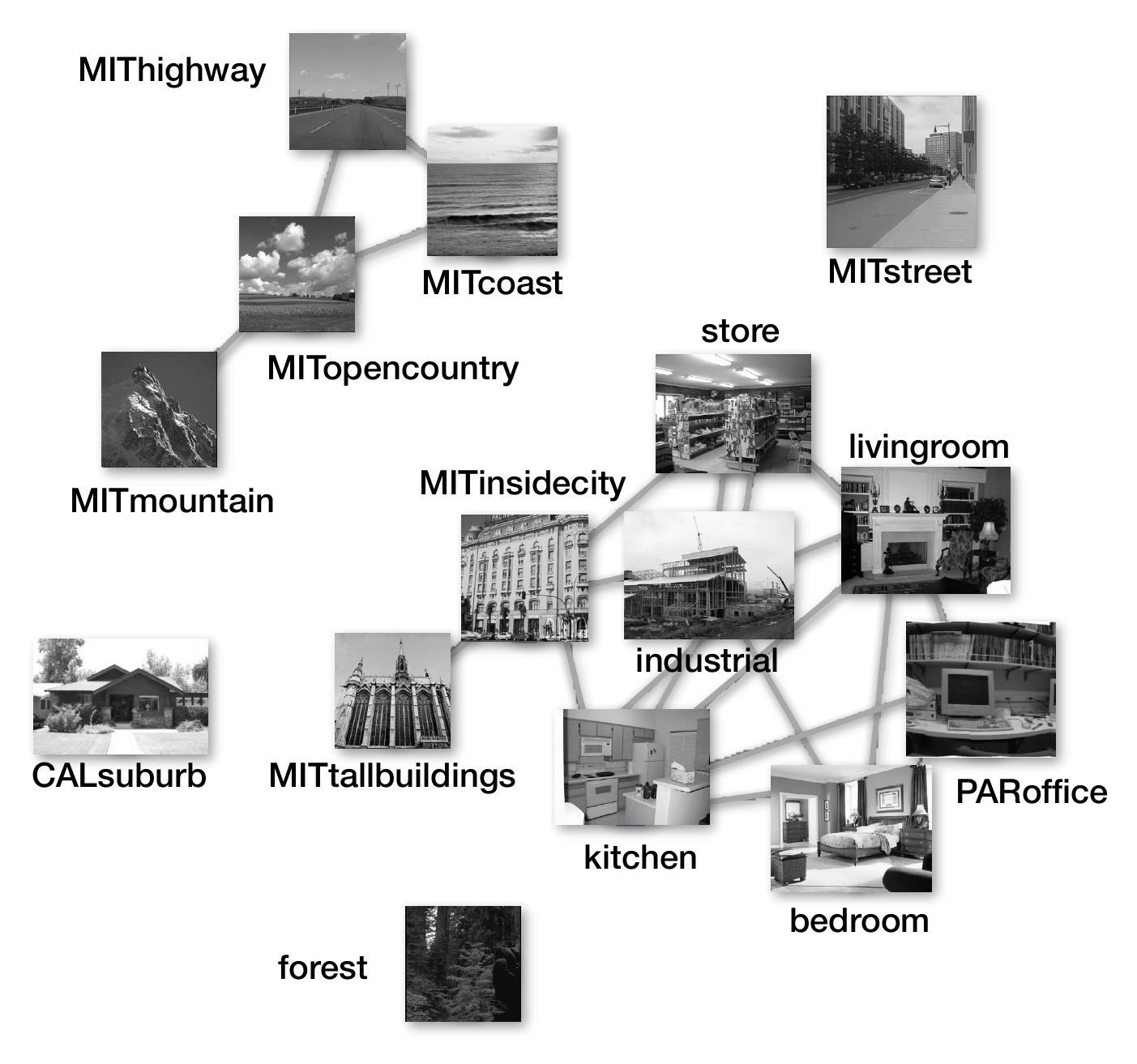}
    \caption{Tasks structure graph recovered by the Sparse Kernel MTL (SKMTL) proposed in this work on the $15$-scenes dataset.}\label{fig:structure}
  \end{center}
\end{figure}

\subsection{Animals with Attributes}\label{sec:awa}

Animals with Attributes\footnote{http://attributes.kyb.tuebingen.mpg.de/} (AwA) is a dataset designed to benchmark detection algorithms in computer vision. The dataset comprises $50$ different animal classes each annotated with $85$ binary labels denoting the presence/absence of different attributes. These attributes can be of different nature such as color (white, black, etc.), texture (stripes, dots), type of limbs (hands, flippers, etc.), diet and so on. The standard challenge is to perform attribute detection by training the system on a predefined set of $40$ animal classes and testing on the remaining $10$. In the following we will first discuss the performance of multi-task approaches in this setting and then investigate how the benefits of multi-task approaches can sometime be dulled by the so-called ``negative transfer'' and how our Sparse Kernel MTL method seems to be less sensitive to such an issue. For the experiments described in the following we used the DECAF features~\cite{Donahue13} recently made available on the Animals With Attribute website.

\subsubsection{Attribute Detection}
We considered the multi-task problem of attribute detection which consists in $85$ classification (binary) tasks. For each attribute, we randomly sampled $50$, $100$ and $150$ examples for training, $500$ for validation and $500$ for test. Results were averaged over $10$ trials. In Table~\ref{tab:awa_complete} we report the Average Precision (area under the precision/recall curve) of the multi-task classifiers tested. As can be noticed for all multi-task approaches, the effect of sharing information across classifiers seems to have a remarkable impact when few training examples are available (the $50$ or $100$ columns in Table~\ref{tab:awa_complete}). As expected, such benefit decreases as the role of regularization becomes less crucial ($150$).

\begin{table}[t]
\begin{center}
\rowcolors{3}{}{gray!35}
\begin{tabular}{lccc}
     & \multicolumn{3}{c}{\bf AUC (\%) per \#tr. samples per class} \tstrut \bstrut \\
    &  50 & 100 & 150 \tstrut \bstrut \\
    \specialrule{.1em}{.05em}{.0em} 

    {\bf STL}       & $57.26 \pm 1.71$  & $60.73 \pm 1.12$  & $64.37 \pm 1.29$ \tstrut \bstrut \\
    {\bf MTFL}      & $58.11 \pm 1.23$  & $61.21 \pm 1.14$  & $64.22 \pm 1.10$ \tstrut \bstrut \\
    {\bf MTRL}      & $58.24 \pm 1.84$  & $61.18 \pm 1.23$  & $64.56 \pm 1.41$ \tstrut \bstrut \\
    {\bf OKL}       & $\mathbf{58.81 \pm 1.18}$  & $62.07 \pm 1.05$  & $64.26 \pm 1.18$ \tstrut \bstrut \\
    {\bf SKMTL}    & $58.63 \pm 1.73$  & $\mathbf{63.21 \pm 1.43}$  & $64.51 \pm 1.83$ \tstrut \bstrut \\

    \end{tabular}
\end{center}
\caption{Attribute detection results on the Animals with Attributes dataset.}
\label{tab:awa_complete}
\end{table}

\subsubsection{Attribute Prediction - Color Vs Limb Shape}
Multi-task learning approaches ground on the assumption that tasks are strongly related one to the other and that such structure can be exploited to improve overall prediction. When this assumption doesn't hold, or holds only partially (e.g. only {\it some} tasks have common structure), such methods could even result disadvantageous (``negative transfer''~\cite{salakhutdinov11}). 


The AwA dataset offers the possibility to observe this effect since attributes are organized into multiple semantic groups~\cite{lampert11,jayaraman14}. We focused on a smaller setting by selecting only two group of tasks, namely {\it color} and {\it limb shape}, and tested the effect of training multi-task methods jointly or independently across such two groups. For all the experiments we randomly sampled for each class $100$ examples for training, $500$ for validation and $500$ for test, averaging the system performance over $10$ trials. Table~\ref{tab:awa_small} reports the average precision separately for the color and limb shape groups. 

Interestingly, methods relying on the assumption that all tasks share a common structure, such as MTFL, MTRL or OKL, experience a slight drop in performance when trained on all attribute detection tasks together (right columns) rather than separately (left column). On the contrary, SKMTL remains stable since it correctly separates the two groups. 


\begin{table}[t]
\begin{center}
\begin{tabular}{lcccc}
    & \multicolumn{4}{c}{\bf Area under PR Curve (\%)} \tstrut \bstrut \\
     & \multicolumn{2}{c}{\bf Independent} & \multicolumn{2}{c}{\bf Joint} \tstrut \bstrut \\
    & {\bf Color} & {\bf Limb} & {\bf Color} & {\bf Limb} \tstrut \bstrut \\
    \specialrule{.1em}{.05em}{.0em} 

    \multirow{2}{*}{\bf STL}        & $74.33$  & $68.13$  & $74.33$ & $68.15$ \tstrut \bstrut \\
                                    & \cellcolor{gray!35} $\pm 0.81 $ & \cellcolor{gray!35} $\pm 0.93 $ & \cellcolor{gray!35} $\pm 0.81 $ & \cellcolor{gray!35} $\pm 0.91 $ \tstrut \bstrut \\

    \multirow{2}{*}{\bf MTFL}      & $75.21$  & $69.41$  & $74.98$  & $69.71$ \tstrut \bstrut \\
                                    & \cellcolor{gray!35} $\pm 0.73 $ & \cellcolor{gray!35} $\pm 1.01 $ & \cellcolor{gray!35} $\pm 1.18 $ & \cellcolor{gray!35} $\pm 0.81 $ \tstrut \bstrut \\
    \multirow{2}{*}{\bf MTRL}      & $75.17$  & $69.18$  & $74.92$  & $69.73$ \tstrut \bstrut \\
                                    & \cellcolor{gray!35} $\pm 0.53 $ & \cellcolor{gray!35} $\pm 0.64 $ & \cellcolor{gray!35} $\pm 0.78 $ & \cellcolor{gray!35} $\pm 0.75 $ \tstrut \bstrut \\
    \multirow{2}{*}{\bf OKL}       & $74.52$  & $68.54$  & $74.31$  & $68.44$ \tstrut \bstrut \\
                                    & \cellcolor{gray!35} $\pm 0.44 $ & \cellcolor{gray!35} $\pm 0.61$ & \cellcolor{gray!35} $\pm 0.54 $ & \cellcolor{gray!35} $\pm 0.22 $ \tstrut \bstrut \\
    \multirow{2}{*}{\bf SKMTL}    & $75.14$  & $69.21$  & $75.23$  & $69.57$ \tstrut \bstrut \\
                                    & \cellcolor{gray!35} $\pm 0.97 $ & \cellcolor{gray!35} $\pm 0.83 $ & \cellcolor{gray!35} $\pm 0.77 $ & \cellcolor{gray!35} $\pm 0.76 $ \tstrut \bstrut \\

    \end{tabular}
\end{center}
\caption{Attribute detection on two subsets of AwA. Comparison between methods trained independently or jointly on the two sets show the effects of negative transfer.}
\label{tab:awa_small}
\end{table}

\section{Conclusions}

We proposed a learning framework designed to solve multiple related tasks while simultaneously recovering their structure. We considered the setting of Reproducing Kernel Hilbert Spaces for vector-valued functions~\cite{micchelli04} and formulated the Sparse Kernel MTL as an output kernel learning problem where both a multi-task predictor and a matrix encoding the tasks relations are inferred from empirical data. We imposed a sparsity penalty on the set of possible relations among tasks in order to recover only those that are more relevant to the learning problem.

Adopting an alternating minimization strategy we were able to devise an optimization algorithm that provably converges to the global solution of the proposed learning problem. Empirical evaluation on both synthetic and real dataset confirmed the validity of the model proposed, which successfully recovered interpretable structures while at the same time outperformed previous methods.

Future research directions will focus mainly on modeling aspects: it will be interesting to investigate the possibility to combine our framework, which identifies sparse relations among the tasks, with recent multi-task linear models that take a different perspective and enforce tasks relations in the form of structured 
sparsity penalties on the feature space~\cite{jayaraman14,zhong12}.

\newpage

{\small
\bibliographystyle{ieee}
\bibliography{egbib}
}

\newpage

\section{Appendix}

\subsection{On the (joint) convexity of Sparse Kernel MTL}

As stated in the paper, it can be shown that the Sparse Kernel MTL problem introduced in Eq.~\eqref{eq:sparse_relation_learning} is jointly convex in the two optimization variables $f$ and $A$. The proof of this fact requires the introduction of functional analysis tools that are beyond the scope of this work. Indeed, according to equation~\eqref{eq:eq} we have observed that it is possible to restrict the SKMTL problem to functions of the form $f(\cdot) = \sum_{i=1}^n k(\cdot,x_i)b_i$ with $b_i\in\mathbb{R}^T$. The following result proves the joint-convexity of Eq.~\eqref{eq:sparse_relation_learning} for this setting. It is an extension of similar results in~\cite{argyriou08,zhang10} and we give it here for completeness. 

\begin{proposition}
Let $V:\mathbb{R}^T \to \mathbb{R}^T \to \mathbb{R}_+$ be a convex loss function. Then the functional in problem~\eqref{eq:sparse_relation_learning} -- restricted to functions $f$ of the form $f(\cdot) = \sum_{i=1}^n k(\cdot,x_i)b_i$ with $b_i\in\mathbb{R}^T$ -- is convex in both $f$ and $A$.
\end{proposition}

\begin{proof}
Notice that, the only term that requires some care is the component of the functional that is mixing $f$ and $A$ together, namely $\|f\|_{\mathcal{H}}$ (where the dependency to $A$ is implicit in $\mathcal{H}$. Indeed, since $V$ is chosen to be convex, the empirical risk term is clearly convex in $f$ and does not depend on $A$, while all the remaining terms are -- i.e. the $tr(A^{-1})$, $tr(A)$ and $\|A\|_{\ell_1}$ -- penalize only the structure matrix $A$ and are clearly convex with respect to it.

According to Eq.~\eqref{eq:eq} $f(\cdot) = \sum_{i=1}^n k(\cdot,x_i)b_i$ and we have that $\|f\|^2_{\mathcal{H}}$ can be rewritten as $\|f\|_{\mathcal{H}}^2 = \tr(B^\top K B A^{-1}))$, with $K \in S_+^n$ the empirical kernel matrix and $B\in\mathbb{R}^{n\times T}$ the matrix whose rows correspond to $b_i^\top$. Let us now set $b = vec(B) \in\mathbb{R}^{nT}$ the vectorization of matrix $B$, obtained by concatenating the columns of $B$. Then we have that
\begin{equation}
tr(B^\top K B A^{-1}) = b^T (A^{-1} \otimes K) b.
\end{equation}
In order to show that the function $Q(A,b) = b^T (A^{-1} \otimes K) b$ is jointly convex in $b$ and $A$ we will show that its epigraph is a convex set. To see this notice that 
\begin{equation}
\begin{aligned}
epi_Q = \{(A,b,c) \in S_{++}^T \times \mathbb{R}^{nT} \times \mathbb{R} \ | \ c \geq w^\top (A^{-1} \otimes K) w \} \\ = \{(A,b,c) \in S_{++}^T \times \mathbb{R}^{nT} \times \mathbb{R} \ | \ \left( \begin{array}{cc}  A \otimes K^{\dagger} & b \\ b^\top & c  \end{array} \right) \in S_+^{nt+1} \}     
\end{aligned}
\end{equation}
where the second equality is directly derived from a Schur's complement argument. Consider now any couple of points $(A_1,b_1,c_1),(A_2,b_2,c_2) \in epi_Q$ and any $\theta \in [0,1]$. We clearly have that the convex combination 
\begin{multline}
    \theta \left( \begin{array}{cc}  A_1 \otimes K^{\dagger} & b_1 \\ b_1^\top & c_1  \end{array} \right) + (1-\theta) \left( \begin{array}{cc}  A_2 \otimes K^{\dagger} & b_2 \\ b_2^\top & c_2  \end{array} \right) \\
    =  \left( \begin{array}{cc}  \theta A_1 \otimes K^{\dagger} + (1-\theta) A_2 \otimes K^{\dagger} & \theta b_1 + (1-\theta) b_2 \\ \theta b_1^\top + (1-\theta) b_2^\top & \theta c_1 + (1-\theta)c_2  \end{array} \right)
\end{multline}
still belongs to $S_+^{nT+1}$, which implies that
\begin{equation}
    (\theta A_1 + (1-\theta) A_2, \theta b_1 + (1-\theta) b_2, \theta c_1 + (1-\theta) c_2) \in epi_Q
\end{equation}
therefore proving that $Q$ is jointly convex in $b$ and $A$.

\end{proof}

\subsection{Cluster Multi-task Learning}

We briefly recall here the Convex Multi-task Cluster Learning proposed in~\cite{jacob08} and show that it can be cast in the same framework as that of our Sparse Kernel MTL model. In particular we comment what choice of constraint set $\mathcal{A}$ can be imposed on the structure matrix $A$ to recover clustered structures of tasks.

In the setting proposed by~\cite{jacob08}, tasks are assumed to belong to one of $r$ of unknown clusters, with $r$ fixed a priori. While the original formulation is for the linear kernel, it can be easily extended to the non-linear setting of RKHSvv. Let $E\in\{0,1\}^{T \times r}$ be the binary matrix whose entry $E_{st}$ has value $1$ whenever a task $s$ belongs to cluster $t$, and $0$ otherwise. 
Let $L$ be the normalized Laplacian of the Graph defined by $E$. Set $M=I - L$, and $U=\frac{1}{T}11^{\top}$. 
As we have observed in Eq.~\eqref{eq:eq}, the regularizer $\|f\|_{\mathcal{H}}$ depends on $A^{-1}$. The role of this term could be shaped to reflect the structure of the clusters encoded in the Laplacian $L$, hence in the matrix $M$.
As noted in ~\cite{jacob08} $A^{-1}(M)$ can be chosen so that:
\begin{equation}
    A^{-1}(M)=\epsilon_{M}U+\epsilon_{B}(M-U)+\epsilon_{W}(I-M),
\end{equation}
where the first term is a global penalty on the average predictor, the second term penalizes the between cluster variance, and the third term penalizes the within cluster variance. Since $M$ belongs to a discrete set, the authors propose a relaxation for $M$ by constraining it to be in a convex set $\mathcal{S}_c=\{ M \in S^{T}_{+} , 0\preceq M \preceq I, \tr(M)=r \}$ which directly induces a set $\mathcal{A}$ of spectral constraints for $A$.

\end{document}